\documentclass[11pt]{article}

\usepackage[T1]{fontenc}
\usepackage{amsmath,amssymb,amsthm,mathtools,mathrsfs}
\usepackage{newtxtext,newtxmath}
\usepackage[margin=1.35in]{geometry}
\usepackage{microtype}
\usepackage{xcolor}
\usepackage{hyperref}
\usepackage{authblk}

\allowdisplaybreaks

\newtheorem{theorem}{Theorem}[section]
\newtheorem{proposition}[theorem]{Proposition}
\newtheorem{lemma}[theorem]{Lemma}

\theoremstyle{definition}
\newtheorem{definition}[theorem]{Definition}
\theoremstyle{remark}
\newtheorem{remark}[theorem]{Remark}

\DeclareMathOperator*{\argmin}{arg\,min}
\DeclareMathOperator*{\argmax}{arg\,max}
\DeclareMathOperator{\IG}{IG}
\newcommand{\KL}{D_{\mathrm{KL}}}
\newcommand{\JS}{D_{\mathrm{JS}}}
\newcommand{\1}{\mathbf 1}

\title{Identifying Intelligent Processes via Online Sequential Testing}
\author[1]{Aritra Das}
\author[1]{Debayan Gupta}
\affil[1]{Truth Audit Labs}
\date{}

\begin{document}

\maketitle

\begin{abstract}
Active sequential hypothesis testing studies how to identify an unknown hypothesis with a given set of sensing actions. We study this in the setting of identifying large language models (LLMs), \textit{i.e.}, if a user is conversing with an LLM drawn from a known set of models, how can they identify which one is in use? Here, the available sensing actions (evaluations) are themselves a design choice: an evaluator must first decide which environments and prompt families to construct, and only then decide how to use them sequentially. We formalize these two levels as an outer probe-design problem and an inner identification problem. Simply put, the outer stage selects a set of probes to be sent to the entire set of models, creating a kind of fingerprint dataset. This is followed by the inner stage, which sequentially sends a budget-minimizing set of those probes to identify the model in use. For the outer problem, we show that selecting which evaluations to construct at minimum cost, so that every pair of candidates is distinguished, is exactly a weighted set cover problem. Since the response distributions of the candidate models are not known exactly but only through calibration samples, we give a one-shot procedure that estimates the cover instance from these samples. For the inner problem, we bound the number of evaluations needed to identify the unknown model in terms of how well the available evaluations distinguish each pair of candidates.

\end{abstract}

\section{Introduction}

Suppose one of $M$ candidate language models is placed behind a black-box
interface.  An evaluator may choose prompts or interactive tasks, observe the
responses, and adapt future evaluations to the data collected so far.  This is
an instance of active sequential hypothesis testing.  The classical
formulation assumes that the sensing actions and their observation laws are
fixed in advance; the problem is to choose actions, a stopping time, and a
final declaration \cite{naghshvar2013active}.

That formulation captures only the second stage of TIPTOE.  Before a
sequential test can be run, one must decide which probes to build and retain.
Constructing an environment, or curating a prompt family
can incur a one-time cost even when the resulting probe is rarely used.  This
choice changes the action set itself and therefore cannot be represented by a
policy that merely selects among already available actions. This leads to two things. First, a  pairwise separation margin gives a guaranty 
of the inner sample complexity

\section{Sequential identification under a fixed design}
\label{sec:fixed-design}

\subsection{Experiment and posterior}

Fix an integer $M\geq2$.  Write $[M]:=\{1,\ldots,M\}$ and let
$\Theta:=\{\theta_i:i\in[M]\}$ be a finite collection of candidate models.
An index $I^\star\in[M]$ is drawn from a prior $b_0\in\Delta_M$, where
\[
    \Delta_M
    :=
    \left\{
        b\in[0,1]^M:
        \sum_{i=1}^M b(i)=1
    \right\},
\]
and $b_0(i)>0$ for every $i$.  All logarithms are natural.

Let $(\mathcal A,\mathscr A)$ be a measurable space of evaluations and let
$(\mathsf Y,\mathcal Y)$ be the response space.  For each $i\in[M]$, let
$P_i$ be a Markov kernel from $(\mathcal A,\mathscr A)$ to
$(\mathsf Y,\mathcal Y)$, and write
\[
    P_i^a(B):=P_i(a,B),
    \qquad a\in\mathcal A,\quad B\in\mathcal Y.
\]
Thus $P_i^a$ is the response law of model $i$ under evaluation $a$.  An
evaluation may be a prompt or an environment--task pair.  A response may be a
completion, an interaction trajectory, or any measurable record of the
interaction.

At time $t\geq1$, the evaluator chooses $A_t$ from the history
\[
    H_{t-1}:=((A_1,Y_1),\ldots,(A_{t-1},Y_{t-1})),
    \qquad H_0:=\varnothing,
\]
and then observes $Y_t$.  Let $\mathcal F_t:=\sigma(H_t)$.  A policy
$\pi=(\pi_t)_{t\geq1}$ is a sequence of probability kernels such that, for
every $C\in\mathscr A$,
\[
    \mathbb P(A_t\in C\mid\mathcal F_{t-1})
    =\pi_t(C\mid H_{t-1})
    \quad\text{almost surely}.
\]

We assume that each evaluation starts from a fresh model context.  Formally,
for every $i\in[M]$ and $B\in\mathcal Y$,
\begin{equation}
    \mathbb P\!\left(
        Y_t\in B
        \mid I^\star=i,\mathcal F_{t-1},A_t
    \right)
    =P_i^{A_t}(B)
    \quad\text{almost surely}.
    \label{eq:reset}
\end{equation}
We write $\mathbb P_{b_0}^{\pi}$ and $\mathbb E_{b_0}^{\pi}$ for probability
and expectation under the induced law.

For $a\in\mathcal A$, let $\mu^a:=\sum_{i=1}^M P_i^a$.  Assume that the
Radon--Nikodym derivatives have jointly measurable versions, and write
\[
    p_i^a(y):=\frac{\mathrm dP_i^a}{\mathrm d\mu^a}(y).
\]
The posterior belief
\[
    b_t(i):=\mathbb P_{b_0}^{\pi}(I^\star=i\mid\mathcal F_t)
\]
satisfies
\begin{equation}
    b_t(i)
    =
    \frac{b_{t-1}(i)p_i^{A_t}(Y_t)}
    {\sum_{j=1}^M b_{t-1}(j)p_j^{A_t}(Y_t)},
    \qquad i\in[M],
    \label{eq:posterior}
\end{equation}
almost surely.  The denominator is positive under the predictive law, so the
update is defined almost surely.

For a belief $b\in\Delta_M$ and evaluation $a\in\mathcal A$, define the
predictive response law
\[
    \overline P_b^a:=\sum_{i=1}^M b(i)P_i^a
\]
and the expected one-step information gain
\begin{equation}
    \IG(b,a)
    :=
    \sum_{i:b(i)>0}
    b(i)\,
    \KL\!\left(P_i^a\,\middle\|\,\overline P_b^a\right).
    \label{eq:information-gain}
\end{equation}
This is the mutual information between $I^\star$ and the next response,
conditional on the current belief and the chosen evaluation
\cite[Chapter~2]{coverthomas2006}.

\subsection{The inner identification problem}

A probe design $\gamma$ determines a measurable set
$\mathcal A_\gamma\subseteq\mathcal A$ of available evaluations.  Let
$\Pi_\gamma$ be the policies supported on $\mathcal A_\gamma$.

\begin{definition}[Fixed-design value]
\label{def:inner-value}
Fix $\delta\in(0,1)$.  An admissible procedure under design $\gamma$ consists
of a policy $\pi\in\Pi_\gamma$, an almost surely finite stopping time $\tau$,
and an $\mathcal F_\tau$-measurable estimate $\widehat I\in[M]$.  Define
\begin{equation}
    V_\delta(\gamma;b_0)
    :=
    \inf_{\pi,\tau,\widehat I}
    \mathbb E_{b_0}^{\pi}[\tau]
    \quad\text{subject to}\quad
    \mathbb P_{b_0}^{\pi}(\widehat I\neq I^\star)\leq\delta,
    \label{eq:inner-objective}
\end{equation}
where the infimum of an empty feasible set is $+\infty$.
\end{definition}

This is the constrained, or primal, version of active sequential hypothesis
testing after the sensing actions and response kernels have been fixed
\cite[Sections~2.1 and~5.3]{naghshvar2013active}.

A standard stopping rule declares a model once its posterior exceeds the
desired confidence level.  Let
\begin{equation}
    \tau_\delta
    :=
    \inf\left\{
        t\geq0:\|b_t\|_\infty\geq1-\delta
    \right\},
    \qquad \inf\varnothing:=\infty,
    \label{eq:posterior-stopping-time}
\end{equation}
and, on $\{\tau_\delta<\infty\}$, let
$\widehat I_\delta\in\argmax_i b_{\tau_\delta}(i)$.

\begin{proposition}[Posterior confidence implies Bayes accuracy]
\label{prop:posterior-confidence}
If $\tau_\delta<\infty$ almost surely, then
\[
    \mathbb P_{b_0}^{\pi}(\widehat I_\delta\neq I^\star)
    =
    \mathbb E_{b_0}^{\pi}\!\left[
        1-\max_{i\in[M]}b_{\tau_\delta}(i)
    \right]
    \leq\delta.
\]
\end{proposition}

\begin{proof}
Conditioned on $\mathcal F_{\tau_\delta}$, a maximum-posterior declaration has
error probability $1-\max_i b_{\tau_\delta}(i)$.  Taking expectations gives
the identity, and the stopping rule gives the inequality.
\end{proof}

\begin{remark}[Greedy information gain]
\label{rem:greedy-ig}
When $\mathcal A_\gamma$ is finite, the myopic rule
\[
    A_{t+1}\in\argmax_{a\in\mathcal A_\gamma}\IG(b_t,a)
\]
is well defined after fixing a tie-breaking rule.  It is one admissible inner
policy; it is not the definition of $V_\delta$, and no optimality claim for
this greedy rule is used below.
\end{remark}
.
\section{Outer probe design}
\label{sec:outer-design}

\subsection{A finite dictionary of modules}

Let $\mathcal U$ be a finite nonempty collection of probe modules.  Each
module $u\in\mathcal U$ has a positive construction cost $\kappa_u$ and makes
a nonempty finite set of evaluations $\mathcal A_u\subseteq\mathcal A$
available.  A design is a subset $S\subseteq\mathcal U$, with
\begin{equation}
    C(S):=\sum_{u\in S}\kappa_u,
    \qquad
    \mathcal A_S:=\bigcup_{u\in S}\mathcal A_u.
    \label{eq:design-cost-actions}
\end{equation}
Let $\gamma_S$ denote the design with action set $\mathcal A_S$.  Construction
cost is paid once, before the unknown model is tested.  Different modules may
share evaluations; module costs are paid separately, while duplicate
evaluations occur only once in $\mathcal A_S$.

For a construction budget $B\geq0$, the exact outer problem is
\begin{equation}
    V_\delta^\star(B;b_0)
    :=
    \min_{S\subseteq\mathcal U:\,C(S)\leq B}
    V_\delta(\gamma_S;b_0).
    \label{eq:exact-outer-objective}
\end{equation}
The minimum is attained because the dictionary is finite, although its value
may be infinite.

\begin{proposition}[Monotonicity]
\label{prop:monotonicity}
If $\mathcal A_S\subseteq\mathcal A_T$, then
\[
    V_\delta(\gamma_T;b_0)
    \leq
    V_\delta(\gamma_S;b_0).
\]
\end{proposition}

\begin{proof}
Every procedure admissible under $\gamma_S$ is also admissible under
$\gamma_T$.
\end{proof}

\begin{remark}[Per-use costs]
If evaluations have different online costs, one may replace $\tau$ in
\eqref{eq:inner-objective} by the accumulated cost
$\sum_{t=1}^{\tau}c(A_t)$.  The construction cost $C(S)$ remains an outer,
one-time cost.
\end{remark}

\subsection{Finite response summaries and pairwise separation}
\label{sec:pairwise-separation}
The inner identification problem uses the full response distributions
$P_i^a$.  For the outer design problem, however, we must compare these
distributions using a finite number of calibration samples.  This is
difficult when responses are long texts or interaction histories, because
the number of possible responses can be very large.  We therefore map each
response to one of $K$ categories using a measurable summary
\[
    \varphi_a:\mathsf{Y}\to[K].
\]
The categories may represent, for example, success or failure, a score
range, or a type of error.  The probability of each category can then be
estimated from its frequency in the calibration samples.  This summary may
remove some differences between models.  Therefore, different summarized
distributions imply different full response distributions, but different
full response distributions need not produce different summarized
distributions.

Let
\[
    \mathcal J:=\{(i,j):1\leq i<j\leq M\},
    \qquad
    m:=|\mathcal J|=\binom M2.
\]
For a module $u$ and pair $(i,j)\in\mathcal J$, define
\begin{equation}
    d_{ij}(u)
    :=
    \max_{a\in\mathcal A_u}
    \|Q_i^a-Q_j^a\|_1.
    \label{eq:module-separation}
\end{equation}
For a design $S$, define its worst-pair margin
\begin{equation}
    \Delta(S)
    :=
    \min_{(i,j)\in\mathcal J}
    \max_{u\in S}d_{ij}(u)
    =
    \min_{(i,j)\in\mathcal J}
    \max_{a\in\mathcal A_S}
    \|Q_i^a-Q_j^a\|_1,
    \label{eq:design-margin}
\end{equation}
with the convention that the maximum over an empty set is zero.  We call $S$
$\eta$-separating if $\Delta(S)\geq\eta$.

It is to be noted that, the margin $\Delta(S)$ is useful as an outer-design objective only if it provides a guarantee for the inner identification problem. In the following theorem, we gives such a guarantee. It converts the pairwise $\ell_1$ separation supplied by an $\eta$-separating design into binary pairwise tests and combines them in a tournament. This leads to an an upper bound on $V_\delta(\gamma_S,b_0)$
\begin{theorem}[Witness bound]
\label{thm:witness-bound}
Let $S$ be $\eta$-separating for some $\eta>0$.  Then, for every prior $b_0$
and every $\delta\in(0,1)$,
\begin{equation}
    V_\delta(\gamma_S;b_0)
    \leq
    (M-1)
    \left\lceil
        \frac{8}{\eta^2}
        \log\frac{M-1}{\delta}
    \right\rceil.
    \label{eq:witness-bound}
\end{equation}
\end{theorem}

\begin{proof}
For every ordered pair $r\neq s$, choose an evaluation $a_{rs}\in\mathcal A_S$
with $\|Q_r^{a_{rs}}-Q_s^{a_{rs}}\|_1\geq\eta$.  For categorical laws,
\[
    \|p-q\|_1
    =2\max_{E\subseteq[K]}\bigl(p(E)-q(E)\bigr).
\]
Hence there is a set $E_{rs}\subseteq[K]$ such that
\[
    Q_r^{a_{rs}}(E_{rs})-Q_s^{a_{rs}}(E_{rs})
    \geq \frac{\eta}{2}.
\]
Let
\[
    n
    :=
    \left\lceil
        \frac{8}{\eta^2}
        \log\frac{M-1}{\delta}
    \right\rceil.
\]
To compare $r$ and $s$, run $a_{rs}$ independently $n$ times and let
$\overline Z$ be the average of
$\1\{\varphi_{a_{rs}}(Y)\in E_{rs}\}$.  Declare $r$ if $\overline Z$ is at
least the midpoint of the two candidate means, and declare $s$ otherwise.
Hoeffding's inequality gives, under either candidate,
\begin{equation}
    \mathbb P(\text{pairwise error})
    \leq
    \exp\!\left(-\frac{n\eta^2}{8}\right).
    \label{eq:pairwise-test-error}
\end{equation}

It remains to combine the pairwise tests without testing all $\binom M2$
pairs.  Start with candidate $1$ as the current winner.  For
$j=2,\ldots,M$, compare the current winner with candidate $j$ and retain the
winner of that comparison.  If the true candidate is $i$, comparisons made
before $i$ enters the tournament are irrelevant.  Once $i$ enters, it remains
the current winner provided that every later comparison involving $i$ is
correct.  There are at most $M-1$ such comparisons.  A union bound and
\eqref{eq:pairwise-test-error} therefore give
\[
    \sup_{i\in[M]}
    \mathbb P_i(\widehat I\neq i)
    \leq
    (M-1)\exp\!\left(-\frac{n\eta^2}{8}\right)
    \leq\delta.
\]
The procedure uses exactly $n(M-1)$ evaluations, which proves
\eqref{eq:witness-bound}.
\end{proof}

\begin{remark}[Information-theoretic interpretation]
The Jensen--Shannon divergence \cite{lin1991} is
\[
    \JS(P,Q)
    :=
    \frac12\KL\!\left(P\middle\|\frac{P+Q}{2}\right)
    +
    \frac12\KL\!\left(Q\middle\|\frac{P+Q}{2}\right).
\]
With natural logarithms, Pinsker's inequality and data processing imply
\begin{equation}
    \frac18\|Q_i^a-Q_j^a\|_1^2
    \leq
    \JS(Q_i^a,Q_j^a)
    \leq
    \JS(P_i^a,P_j^a).
    \label{eq:js-relation}
\end{equation}
See, for example, \cite[Lemma~2.5]{tsybakov2009} for Pinsker's inequality and
\cite[Chapter~2]{coverthomas2006} for data processing.  Thus an
$\eta$-separating design provides, for each pair, a retained evaluation with
raw Jensen--Shannon divergence at least $\eta^2/8$.
\end{remark}

\begin{remark}[What the margin does and does not prove]
The margin $\Delta(S)$ gives a universal upper bound on the inner value through
Theorem~\ref{thm:witness-bound}.  It does not determine that value.  The exact
sequential problem depends on the prior, the full response laws, asymmetric
KL divergences, and the possibility of combining evaluations adaptively.
\end{remark}

\subsection{The fixed-margin design problem}
\label{sec:set-cover}

Fix $\eta\in(0,2]$.  For each module $u$, define the candidate pairs it
separates at margin $\eta$:
\begin{equation}
    \mathcal C_u(\eta)
    :=
    \{(i,j)\in\mathcal J:d_{ij}(u)\geq\eta\}.
    \label{eq:cover-set}
\end{equation}

\begin{proposition}[Exact reduction to weighted set cover]
\label{prop:set-cover}
A design $S\subseteq\mathcal U$ is $\eta$-separating if and only if
\begin{equation}
    \bigcup_{u\in S}\mathcal C_u(\eta)=\mathcal J.
    \label{eq:cover-equivalence}
\end{equation}
Consequently, a minimum-cost $\eta$-separating design solves
\begin{equation}
    S_\eta^\star
    \in
    \argmin_{S\subseteq\mathcal U} C(S)
    \quad\text{subject to}\quad
    \bigcup_{u\in S}\mathcal C_u(\eta)=\mathcal J.
    \label{eq:min-cost-cover}
\end{equation}
If no cover exists, set $C(S_\eta^\star):=+\infty$.
\end{proposition}

\begin{proof}
By definition, $\Delta(S)\geq\eta$ exactly when every pair belongs to the
cover set of at least one selected module.
\end{proof}

An exact integer-programming formulation uses a binary variable $x_u$ for each
module:
\begin{equation}
    \begin{aligned}
        \text{minimize}\quad
            &\sum_{u\in\mathcal U}\kappa_u x_u,\\
        \text{subject to}\quad
            &\sum_{u:(i,j)\in\mathcal C_u(\eta)}x_u\geq1,
            &&(i,j)\in\mathcal J,\\
            &x_u\in\{0,1\},
            &&u\in\mathcal U.
    \end{aligned}
    \label{eq:set-cover-ilp}
\end{equation}

The equivalent budgeted margin is
\begin{equation}
    \Delta^\star(B)
    :=
    \max_{S\subseteq\mathcal U:\,C(S)\leq B}\Delta(S).
    \label{eq:budgeted-margin}
\end{equation}
For every $\eta$,
\begin{equation}
    \Delta^\star(B)\geq\eta
    \quad\Longleftrightarrow\quad
    C(S_\eta^\star)\leq B.
    \label{eq:margin-cost-equivalence}
\end{equation}
Because the dictionary is finite, $\Delta^\star(B)$ belongs to the finite set
$\{0\}\cup\{d_{ij}(u):(i,j)\in\mathcal J,u\in\mathcal U\}$.  Sorting these
values and solving the exact cover feasibility problem recovers
$\Delta^\star(B)$.  An approximate cover solver instead gives a relaxed
budget guarantee, as described below.

\section{Learning the cover instance from calibration samples}
\label{sec:calibration}

Let
\begin{equation}
    \mathcal A_0:=\bigcup_{u\in\mathcal U}\mathcal A_u,
    \qquad
    A_0:=|\mathcal A_0|.
    \label{eq:action-pool}
\end{equation}
During calibration, assume independent sample access to every candidate under
every evaluation in $\mathcal A_0$.  Thus, for each $(i,a)$, one may draw
fresh samples from $Q_i^a$.  These samples come from the known candidate
models and are separate from the later interaction with the unknown model.

A sample from one candidate--evaluation pair is reused in all pairwise scores
involving that candidate and in all modules containing that evaluation.  We index the samples by
$[M]\times\mathcal A_0$.

\subsection{Two margins}

A finite-sample method cannot uniformly decide whether a score is at least a
single threshold when the true score may lie arbitrarily close to that
threshold.  Fix
\begin{equation}
    0<\eta_-<\eta_+\leq2,
    \qquad
    g:=\eta_+-\eta_-.
    \label{eq:two-margins}
\end{equation}
The goal is to construct empirical cover sets $\widehat{\mathcal C}_u$ such
that
\begin{equation}
    \mathcal C_u(\eta_+)
    \subseteq
    \widehat{\mathcal C}_u
    \subseteq
    \mathcal C_u(\eta_-)
    \qquad\text{for every }u.
    \label{eq:cover-containment}
\end{equation}
The right inclusion guarantees that every retained pair is truly separated at
margin $\eta_-$.  The left inclusion guarantees that no pair separated at the
larger margin $\eta_+$ is lost.  Scores in $[\eta_-,\eta_+)$ may be classified
either way.  Let $\beta\in(0,1)$ be the calibration failure probability; it is
separate from the online identification error $\delta$.

\subsection{Uniform one-shot calibration}

Set
\begin{equation}
    \Lambda
    :=
    K\log2+\log\!\left(\frac{2mA_0}{\beta}\right),
    \qquad
    n
    :=
    \left\lceil\frac{16\Lambda}{g^2}\right\rceil.
    \label{eq:one-shot-n}
\end{equation}
Draw $n$ samples from every candidate--evaluation pair and let
$\widehat Q_i^a$ be the empirical categorical law.  Define
\begin{equation}
    \widehat d_{ij}(u)
    :=
    \max_{a\in\mathcal A_u}
    \|\widehat Q_i^a-\widehat Q_j^a\|_1
    \label{eq:empirical-module-score}
\end{equation}
and
\begin{equation}
    \widehat{\mathcal C}_u
    :=
    \left\{
        (i,j)\in\mathcal J:
        \widehat d_{ij}(u)
        \geq
        \frac{\eta_-+\eta_+}{2}
    \right\}.
    \label{eq:empirical-cover}
\end{equation}

\begin{lemma}[Uniform score accuracy]
\label{lem:uniform-score}
With probability at least $1-\beta$,
\begin{equation}
    |\widehat d_{ij}(u)-d_{ij}(u)|
    \leq\frac g2
    \label{eq:score-error}
\end{equation}
simultaneously for every $(i,j)\in\mathcal J$ and $u\in\mathcal U$.
\end{lemma}

\begin{proof}
For categorical laws $p,q$ on $[K]$,
\begin{equation}
    \|p-q\|_1
    =
    2\max_{E\subseteq[K]}\bigl(p(E)-q(E)\bigr).
    \label{eq:l1-event-identity}
\end{equation}
Fix $i<j$, $a\in\mathcal A_0$, and $E\subseteq[K]$.  Write
\[
    W:=Q_i^a(E)-Q_j^a(E),
    \qquad
    \widehat W:=\widehat Q_i^a(E)-\widehat Q_j^a(E).
\]
The variable $\widehat W-W$ is a sum of $2n$ independent centered terms, each
with range length $1/n$.  Hoeffding's inequality \cite{hoeffding1963} gives
\[
    \mathbb P\!\left(
        |\widehat W-W|>\frac g4
    \right)
    \leq
    2\exp\!\left(-\frac{ng^2}{16}\right).
\]
There are $mA_0 2^K$ choices of $(i,j,a,E)$.  By
\eqref{eq:one-shot-n}, a union bound shows that all witness errors are at most
$g/4$ with probability at least $1-\beta$.  On this event,
\eqref{eq:l1-event-identity} gives
\[
    \left|
        \|\widehat Q_i^a-\widehat Q_j^a\|_1
        -
        \|Q_i^a-Q_j^a\|_1
    \right|
    \leq\frac g2
\]
for every pair and evaluation.  Taking a maximum over the evaluations in a
module preserves the same error bound.
\end{proof}

\begin{theorem}[One-shot two-margin guarantee]
\label{thm:one-shot}
Let the cover solver have approximation factor $\alpha\geq1$.  With
probability at least $1-\beta$, the empirical sets in
\eqref{eq:empirical-cover} satisfy \eqref{eq:cover-containment}.  On the same
event:
\begin{enumerate}
    \item every design that covers the empirical pair sets is
    $\eta_-$-separating;
    \item if an $\eta_+$-separating design exists, the solver returns a design
    $\widehat S$ satisfying
    \begin{equation}
        C(\widehat S)
        \leq
        \alpha C(S_{\eta_+}^\star);
        \label{eq:one-shot-cost}
    \end{equation}
    \item if the empirical cover instance is infeasible, then no
    $\eta_+$-separating design exists;
    \item every returned empirical cover $\widehat S$ satisfies
    \begin{equation}
        V_\delta(\gamma_{\widehat S};b_0)
        \leq
        (M-1)
        \left\lceil
            \frac{8}{\eta_-^2}
            \log\frac{M-1}{\delta}
        \right\rceil.
        \label{eq:one-shot-online-bound}
    \end{equation}
\end{enumerate}
The calibration stage uses exactly $nMA_0$ samples.
\end{theorem}

\begin{proof}
Work on the event in Lemma~\ref{lem:uniform-score}.  If
$d_{ij}(u)\geq\eta_+$, then
\[
    \widehat d_{ij}(u)
    \geq
    \eta_+-\frac g2
    =
    \frac{\eta_-+\eta_+}{2},
\]
so $(i,j)\in\widehat{\mathcal C}_u$.  Conversely, if
$(i,j)\in\widehat{\mathcal C}_u$, then
\[
    d_{ij}(u)
    \geq
    \frac{\eta_-+\eta_+}{2}-\frac g2
    =\eta_-.
\]
This proves \eqref{eq:cover-containment}.

Any empirical cover is therefore a true cover at margin $\eta_-$, proving the
first claim.  The left inclusion in \eqref{eq:cover-containment} makes
$S_{\eta_+}^\star$ feasible for the empirical instance.  Hence the empirical
optimum costs at most $C(S_{\eta_+}^\star)$, which gives
\eqref{eq:one-shot-cost}.  If the empirical instance is infeasible, some pair
belongs to no empirical cover set; the left inclusion then shows that this
pair belongs to no $\mathcal C_u(\eta_+)$.  The final claim follows from
Theorem~\ref{thm:witness-bound}.
\end{proof}

\end{document}